\documentclass{article}

\usepackage[preprint]{neurips_2025}

\usepackage[utf8]{inputenc} 
\usepackage[T1]{fontenc}    
\usepackage{hyperref}       
\usepackage{url}            
\usepackage{booktabs}       
\usepackage{amsfonts}       
\usepackage{nicefrac}       
\usepackage{microtype}      
\usepackage{xcolor}         
\usepackage{algorithm}
\usepackage{algpseudocode}
\usepackage{amsmath}
\usepackage{amssymb}
\usepackage{amsthm}

\usepackage{subcaption}
\usepackage{graphicx} 

\newtheorem{theorem}{Theorem}  

\DeclareMathOperator*{\argmin}{arg\,min}

\title{Beyond First-Order: Training LLMs with Stochastic Conjugate Subgradients and AdamW}

\author{%
  Di Zhang \\
  Department of Industrial and System Engineering\\
  University of Southern California\\
  Los Angeles, CA 90089 \\
  \texttt{dzhang22@usc.edu} \\
  \And
  Yihang Zhang \\
  Department of Industrial and System Engineering\\
  University of Southern California\\
  Los Angeles, CA 90089 \\
  \texttt{zhangyih@usc.edu} \\
}

\begin{document}

\maketitle

\begin{abstract}
   Stochastic gradient-based descent (SGD), have long been central to training large language models (LLMs). However, their effectiveness is increasingly being questioned, particularly in large-scale applications where empirical evidence suggests potential performance limitations. In response, this paper proposes a stochastic conjugate subgradient method together with adaptive sampling tailored specifically for training LLMs. The method not only achieves faster convergence per iteration but also demonstrates improved scalability compared to traditional SGD techniques. It leverages sample complexity analysis to adaptively choose the sample size, employs a stochastic conjugate subgradient approach to determine search directions and utilizing an AdamW-like algorithm to adaptively adjust step sizes. This approach preserves the key advantages of first-order methods while effectively addressing the nonconvexity and non-smoothness inherent in LLMs training. Additionally, we provide a detailed analysis of the advantage of the algorithm. Experimental results show that the proposed method not only maintains, but in many cases surpasses, the scalability of traditional SGD techniques, significantly enhancing both the speed and accuracy of the optimization process.
\end{abstract}

\section{Introduction}

\subsection{Problem Setup}

Large Language Models (LLMs) are trained to predict the next token in a sequence, given the preceding tokens. The training objective is defined as minimizing the Negative Log-Likelihood (NLL) of the training data. Given a dataset $\mathcal{D} = \{(x^{(i)}, y^{(i)})\}_{i=1}^{N}$, where $x^{(i)}$ is the input text sequence and $y^{(i)}$ is the target token sequence, the objective function can be written as:

\begin{equation}
L(\theta) = - \sum_{i=1}^{N} \sum_{t=1}^{T_i} \log P_\theta(y_t^{(i)} \mid y_{1:t-1}^{(i)})
\end{equation}

where:
\begin{itemize}
    \item $T_i$ is the length of the $i$-th sequence.
    \item $y_{1:t-1}^{(i)}$ represents the tokens preceding token $y_t^{(i)}$.
    \item $P_\theta$ is the model's predicted probability distribution over the vocabulary, parameterized by $\theta$.
\end{itemize}

To train the LLMs, many stochastic gradient-based descent (SGD) algorithms have been proposed such as Adam and AdamW. SGD methods dominate the field for convincing reasons: low computational cost, simplicity of implementation, and strong empirical results. However, despite the advantages, there are also many limitations: a) SGD methods are not known for their numerical accuracy especially when the functions are non-smooth and b) establishing rigorous convergence guarantees is highly challenging because of the non-smooth and non-convex property of the objective function.

\subsection{Contributions}

Instead of relying solely on the SGD approaches, a stochastic conjugate subgradient (SCS)~\cite{ZS2024, zhang2024sampling, zhang2024stochastic, zhang2025stochastic} method together with adaptive sampling strategy is considered in this paper. Originated from~\cite{W1975,sen2022stochastic}, the method accommodates the curvature of a objective functions, inherits the spirit of the momentum methods and apply the sample complexity analysis to reduce the computational burden. This combination enhances the power of SGD approaches without the additional burden of second-order approximations. In other words, our method shares the spirit of Hessian-Free (HF) optimization, which has gained attention in the machine learning community~\cite{MJ2010, CScheinberg2017}. As our theoretical analysis and computational results reveal, the new method provides a ``sweet spot" at the intersection of speed, accuracy, and scalability of algorithms. Our main contributions are:

\begin{itemize}
    \item To address the challenges posed by the non-smoothness of LLMs, we extend the non-smooth SCS method~\cite{W1975,zhang2024stochastic} to the training of LLMs. The SCS technique exhibits behavior similar to that of higher-order methods but avoids the significant computational overhead typically associated with maintaining a Hessian matrix at each iteration. To the best of our knowledge, this is the first attempt to apply such a beyond-first-order optimization technique to the training of LLMs.
    
    \item While a straightforward application of~\cite{W1975} to LLMs using sample average approximation (SAA)~\cite{Shapiro2003} might be considered, our focus is on solving LLMs with extremely large datasets. Consequently, we employ an adaptive sampling strategy over a deterministic finite sum approximation. By leveraging the sample complexity analysis, we can take advantage of an adaptive sampling strategy which dramatically reduce the computational burden of training LLMs.   

    \item The computational results demonstrate that, from an optimization standpoint, the solutions produced by SCS yield lower objective function values compared with SGD methods. More crucially, the efficiency and accuracy of our algorithm surpasses those of algorithms such as Adam and AdamW, which are specialized SGD algorithms for training LLMs. 

    \item The proposed algorithm integrates several disparate notions such as adaptive sampling, decomposition, conjugate subgradients, and dynamic learning rate in Adam and AdamW. We demonstrate that this amalgamation effectively approximates solutions for LLMs with extremely large datasets.
\end{itemize}

The structure of this paper is as follows. In \textsection\ref{related work}, we will discuss some related work in the literature. In \textsection\ref{method}, we propose a SCS + AdamW algorithm which integrated the AdamW algorithm with SCS first moment direction. In \textsection\ref{discussion}, some good properties of the SCS direction are discussed. Then in \textsection\ref{experiments}, we present preliminary computational results that demonstrate the advantages of the proposed method over baseline approaches. Next, in \textsection\ref{limitations}, we discuss the limitations of this paper and some work directions in the future. Finally, conclusions are summarized in  \textsection\ref{conclusion}.

\section{Related Work} \label{related work}

\subsection{SGD, Adam and AdamW}
Gradient-based optimization methods have evolved to handle complex, high-dimensional problems in deep learning and machine learning. Kingma and Ba~\cite{kingma2014adam} introduced Adam, an optimization algorithm that combines the advantages of Momentum and RMSProp. It maintains two moving averages for each parameter: (a) First moment estimate (mean of past gradients) (b) Second moment estimate (uncentered variance of past gradients). However, the major limitation of it is that the L2 regularization interacts with adaptive learning rates, making weight decay less effective. Loshchilov and Hutter~\cite{loshchilov2017decoupled} proposed AdamW, addressing the key limitation of Adam. Empirical results show that AdamW improves performance on NLP and vision tasks compared to Adam.

\subsection{Stochastic Conjugate Subgradient Method}

In contrast to first-order methods, algorithms based on conjugate gradients (CG) provide finite termination when the objective function is a linear-quadratic deterministic function~\cite{N2006}. Furthermore, CG algorithms are also used for large-scale nonsmooth convex optimization problems~\cite{W1975,Y2016}. However, despite its many strengths, (e.g., fast per-iteration
convergence, frequent explicit regularization on step-size, and better parallelization than first-order
methods), the use of conjugate gradient methods for Stochastic Optimization is uncommon. Although~\cite{J2018,Y2022} proposed some new stochastic conjugate gradient (SCG) algorithms, their assumptions do not support non-smooth objective functions, and hence not directly applicable to LLMs.

\subsection{Adaptive sampling}
The LLMs algorithm typically employs sample average approximation (SAA)~\cite{S1998} which has a fixed sample size to estimate the uncertainty. However, it is difficult to determine the sample size~\cite{S2016,DS2024} when applying such methods. While  adaptive sampling~\cite{M1999} has emerged as a powerful approach for instances with a large number of outcomes (especially for multi-dimensional uncertainty, see~\cite{HS1991, HS1994,lin2024economic}), such adaptive sampling has remained elusive in the context of LLMs. Our algorithm will strategically sample a subset of scenario subproblems in each iteration, facilitating dynamic scenario updates to enhance computational efficiency. We will also leverage sample complexity analysis~\cite{V1998} to support the choice of the ultimate sample size used at termination. This approach ensures a theoretically grounded, and computationally realistic approach.

This paper propose an integrated framework that combines the SCS with AdamW. Our goal is to create an efficient and rubust algorithm that converges faster than SGD-based algorithms when training LLMs. Figure~\ref{fig: flowchart} provides an overview of the proposed framework.

\begin{figure}[ht]
    \caption{SCSAdamW Flowchart}
    \centering
    \includegraphics[width=0.8\linewidth]{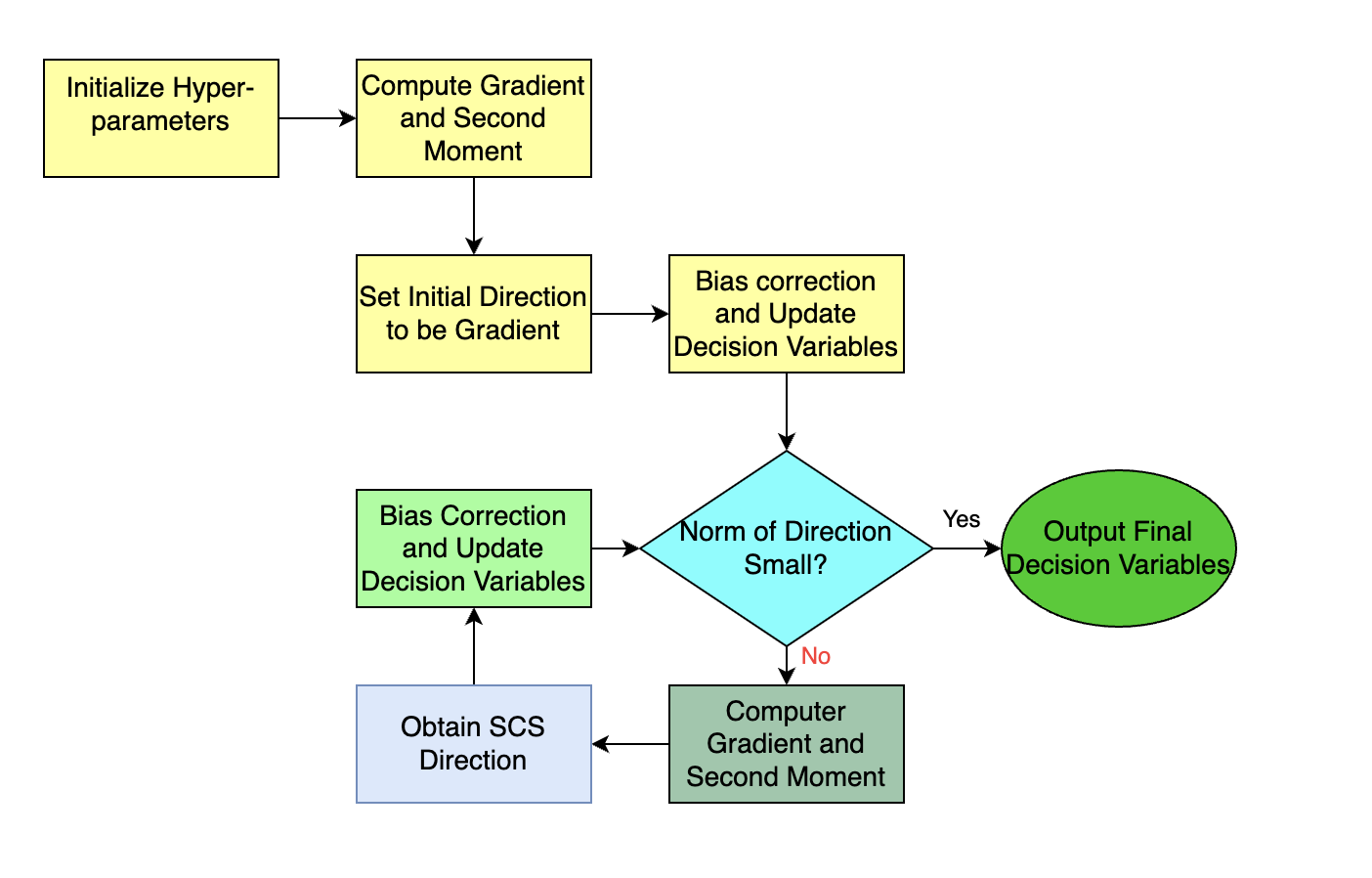}
    \label{fig: flowchart}
\end{figure}

\section{Method} \label{method}

The stochastic conjugate subgradient + AdamW (SCSAdamW) algorithm is a hybrid optimization method that integrates the AdamW optimizer with a conjugate subgradient approach to improve convergence. The method combines adaptive learning rates (AdamW) with conjugate subgradient updates, allowing for better handling of curvature information while maintaining stability and convergence properties. Here’s a summary of its major steps:

\begin{itemize}

    \item \textbf{Sequential adaptive sampling} At iteration $t$, different from classic LLMs algorithms which solves an ''all-in-one`` optimization problem, we will randomly sample $|N_t|$ subproblems. This is similar to using sample average approximation (SAA) to approximate function $L$ using $L_t$,  
    \begin{equation} \label{OSAA}
        L_t(\theta)= \sum_{n=1}^{N_t} \sum_{t=1}^{T_i} \log P_\theta(y_t^{(i)} \mid y_{1:t-1}^{(i)}),
    \end{equation}
    \noindent where $|N_t|$ will be determined based on concentration inequalities in Theorem \ref{sample complexity}.

    \item \textbf{Direction finding} The idea here is inspired by non-smooth conjugate subgradient method \cite{W1974} which uses the smallest norm of the convex combination of the previous search direction ($d_{t-1}$) and the current subgradient ($g_t$). More specifically, we first solve the following one-dimensional QP
    \begin{equation} \label{p^k}
        \lambda_t^* = \argmin_{\lambda_t \in [0,1]}\frac{1}{2}||(1-\lambda_t) \cdot d_{t-1} + \lambda_t \cdot g_t||^2.
    \end{equation}

    \noindent The solution of \eqref{p^k} is,

    \begin{equation*}
    \label{ssn}
        \lambda_t^*=\Pi_{[0,1]}\big(\frac{- \ \langle d_{t-1}, g_t \rangle+||g_t||^2}{||d_{t-1}||^2-2\langle d_{t-1}, g_t \rangle +||g_t||^2}\big),
    \end{equation*}

    \noindent where $\Pi_{[0,1]}$ denotes the projection operator on $[0,1]$. Then we can set the new search direction as
    
    \begin{equation*}
        d_t=(1-\lambda_t^*) \cdot d_{t-1}+\lambda_t^* \cdot g_t,
    \end{equation*}

    \noindent where $G_k=\{d_{k-1},g_k\}$, and $\lambda_k^*$ denotes the optimal weight for the convex combination.  Clearly if one fixes $\lambda_k = 0$, then the search direction reduces to that of the subgradient method.

    \item \textbf{Dynamic learning rate} SCS-AdamW follows an adaptive learning rate scheme, meaning each parameter has its own learning rate. This is achieved by normalizing the conjugate subgradient updates using second momentum estimates. The term $\frac{\eta}{\sqrt{\hat{v}_t} + \varepsilon}$ ensures that the update magnitude adapts to the scale of gradients, preventing large updates in high-gradient regions and avoids vanishing updates in low-gradient regions.

    \item \textbf{Decoupled weight decay} In standard Adam (without decoupled weight decay), L2 regularization is applied inside the subgradient update, meaning the weight decay term is included in the computation of $g_t$. However, this can interfere with the adaptive learning rate mechanism of Adam, making weight decay behave inconsistently across different parameters.AdamW enforces this by subtracting $\eta \lambda \theta_{t-1}$ after the Adam update step, ensuring that:
    (a) Weight decay acts as a separate force reducing parameter values without altering gradient estimates. (b) Improved generalization, since weight decay is applied consistently across parameters, avoiding interference with Adam’s adaptive updates.

    \item \textbf{Termination criteria} The algorithm concludes its process when $||d_t|| \leq \varepsilon$. As we will show in Theorem \ref{d_k and g_k}, a diminutive value of $||d_t||$ indicates a small norm of the subgradient $g_t$, fulfilling the stationary point condition for an unconstrained non-convex optimization problem.

\end{itemize}

\begin{algorithm}
\caption{Stochastic Conjugate Subgradient + AdamW (SCSAdamW) Algorithm}
\begin{algorithmic}[1]
\Require Parameters $\theta$, learning rate $\eta$, decay rates $\beta_2$, weight decay $\lambda$, small constant $\zeta$
\State Initialize $m = 0$, $v = 0$, $d = 0$, $t = 0$

\While{$||d_k|| > \varepsilon$}
    \State $t \leftarrow t + 1$
    \State Compute gradient $g_t = \nabla L_t(\theta_t)$

    \If{$t > 1$}
        \State Compute conjugate direction:
        \[
        \lambda_t^*=\Pi_{[0,1]}\big(\frac{- \ \langle d_{t-1}, g_t \rangle+||g_t||^2}{||d_{t-1}||^2-2\langle d_{t-1}, g_t \rangle +||g_t||^2}\big)
        \]
        \[
        d_t = (1-\lambda_t^*) \cdot d_{t-1}+\lambda_t^* \cdot g_t
        \]
    \Else
         \[d_t = g_t\]
    \EndIf
    \State Update second moment (variance):
    \[
    v_t = \beta_2 v_{t-1} + (1 - \beta_2) g_t^2
    \]

    \State Bias correction:
    \[
    \hat{d}_t = \frac{d_t}{1 - (\lambda_t^*)^t}, \quad \hat{v}_t = \frac{v_t}{1 - \beta_2^t}
    \]

    \State Update parameters:
    \[
    \theta_t = \theta_{t-1} - \frac{\eta}{\sqrt{\hat{v}_t} + \zeta} \hat{d}_t
    \]

    \State Apply decoupled weight decay (AdamW):
    \[
    \theta_t \leftarrow \theta_t - \eta \lambda \theta_{t-1}
    \]
\EndWhile

\end{algorithmic}
\end{algorithm} 

\section{Discussion} \label{discussion}

\subsection{Comparison with Momentum Method}

In this section, we compare our direction-finding method, expressed as
$\alpha_{t+1} = \alpha_{t} - t \cdot d_t$ with several momentum methods and elucidate their relationships. The primary insight is that the direction $-d_t$ represents a specific convex combination of all prior subgradients. Among all possible directions within the convex hull formed by these subgradients, $-d_t$ possesses the smallest norm. While other momentum methods also utilize a specialized convex combination of past subgradients, they do not achieve the minimal norm. Given that the optimality condition for the convex optimization problem is $0 \in \partial f(\alpha_t)$, our algorithm is likely to outperform other momentum methods. 

Polyak’s momentum~\cite{polyak1964some}, also known as the “heavy ball method”, introduces an additional “momentum” term $ \theta \cdot (\alpha_t - \alpha_{t-1})$ The full momentum update is:
\begin{equation*}
    \begin{aligned}
        \alpha_{t+1} & = \alpha_{t} - t \cdot g_t + \theta \cdot (\alpha_t - \alpha_{t-1}) \\
        & = \alpha_{t} - t \cdot g_t + \theta \cdot (- t g_{t-1} + \theta (\alpha_{t-1} - \alpha_{t-2}))\\
        & = \alpha_{t} - t (g_t + \theta g_{t-1} + ... + \theta^{t-1} g_1).
    \end{aligned}
\end{equation*}

\noindent Let $\Theta_t = 1 + \theta + ... \theta^{t-1} = \frac{1 - \theta^t}{1-\theta}$ and $t_t = \frac{t}{\Theta_t}$, we have

\begin{equation*}
    \alpha_{t+1} = \alpha_{t} - t_t \cdot (\frac{1}{\Theta_t} g_t + \frac{\theta}{\Theta_t} g_{t-1} + ... + \frac{\theta^{t-1}}{\Theta_t} g_1).    
\end{equation*}

 \noindent Note that the direction $(\frac{1}{\Theta_t} g_t + \frac{\theta}{\Theta_t} g_{t-1} + ... + \frac{\theta^{t-1}}{\Theta_t} g_1)$ is a convex combination of all previous subgradients and according to Cauchy–Schwarz inequality, Polyak’s momentum will be equivalent to our direction-finding approach if and only if $||g_t||$ decreases exponentially with rate $\theta$, i.e.,

 \begin{equation*}
     ||g_t|| = \theta \cdot ||g_{t-1}|| = ... = \theta^{t-1} \cdot ||g_{1}||.
 \end{equation*}

\subsection{Comparison of Adam, AdamW and SCSAdamW}

Adam is a first-order optimizer that smooths gradient noise using momentum and variance, but can suffer from over-regularization issues due to coupled weight decay. AdamW improves Adam by decoupling weight decay, which significantly improves generalization and training stability.
SCSAdamW further upgrades AdamW by:
Replacing plain gradient direction with a stochastic conjugate subgradient direction. This mimics second-order behavior without explicitly computing or storing a Hessian. It retains the low memory and computation cost of AdamW, but gives stronger search directions, which can speed up convergence or stabilize non-smooth training of LLMs especially with non-smooth loss. Moreover, it also leverages the sample complexity analysis and concentration inequality to determine the training sample size, which the algorithm more robust.  Table \ref{tab:adam-comparison} gives a detailed comparison:

\begin{table}[h]
\centering
\caption{Comparison between Adam, AdamW, and SCSAdamW}
\renewcommand{\arraystretch}{1.2} 
\begin{tabular}{|p{3cm}|p{3cm}|p{3cm}|p{3.5cm}|}
\hline
\textbf{Aspect} & \textbf{Adam} & \textbf{AdamW} & \textbf{SCSAdamW} \\
\hline
Search Direction & Gradient $g_t$ & Gradient $g_t$ & Conjugate subgradient $d_t$ (combining $d_{t-1}$ and $g_t$) \\
\hline
First Moment (Momentum) & Yes ($m_t$) & Yes ($m_t$) & (uses conjugate direction $d_t$) \\
\hline
Second Moment (Variance) & Yes ($v_t$) & Yes ($v_t$) & Yes ($v_t$) \\
\hline
Weight Decay Handling & Coupled (inside gradient update) & Decoupled (applied separately) & Decoupled (applied separately) \\
\hline
Bias Correction & On $m_t$, $v_t$ & On $m_t$, $v_t$ & On $d_t$, $v_t$ \\
\hline
Update Rule & $\theta_t \leftarrow \theta_{t-1} - \eta \frac{\hat{m}_t}{\sqrt{\hat{v}_t} + \zeta}$ & Same as Adam + decoupled weight decay & $\theta_t \leftarrow \theta_{t-1} - \eta \frac{\hat{d}_t}{\sqrt{\hat{v}_t} + \zeta}$ plus weight decay \\
\hline
Higher-Order Information & No & No & Yes (via conjugate subgradient) \\
\hline
Computational Overhead & Low & Low &  slightly more than AdamW \\

\hline
Sample complexity analysis to determine the sample size & No & No &  Yes \\

\hline
Expected Behavior & Smooths noise, slow near saddles & Improves generalization and stability & Faster convergence in non-smooth problems \\
\hline
\end{tabular}
\label{tab:adam-comparison}
\end{table}

\subsection{Stationary Point}
The primary goal of our optimization algorithm is to find a point where the true gradient (or subgradient, for non-smooth functions) of the loss function is close to zero, indicating a stationary point. Our algorithm, SCSAdamW, terminates when the norm of its computed search direction $||d_t||$, falls below a predefined small threshold $\varepsilon$, which is guaranteed by Theorem \ref{smallest norm}. A natural question is whether this internal termination criterion reliably corresponds to the actual subgradient $||g_t||$ also being small. Theorem \ref{d_k and g_k} below addresses this by establishing a formal connection. It demonstrates that if the search direction norm $||d_t||$ newly drops below $\varepsilon$ (having been larger in the previous step), and if the current subgradient $g_t$ maintains some general agreement with (i.e., is not strongly opposing) the previous search direction $d_{t-1}$, then the norm of the true subgradient $||g_t||$ is indeed bounded by a value proportional to $\varepsilon$.

\begin{theorem}
\label{smallest norm}
Let the sequence of search directions $\{d_t\}$ be generated starting with $d_1 = g_1$, where $g_1$ is a subgradient of the objective function at the initial point. For $t > 1$, let $d_t = \arg\min_{x \in \text{conv}\{d_{t-1}, g_t\}} ||x||$.

The following properties hold:
\begin{enumerate}
    \item $d_t \in \text{conv}\{g_1, \dots, g_t\}$ for all $t \geq 1$.
    \item Furthermore, if the subgradients are uniformly bounded such that $||g_t|| \leq C$ for some constant $C > 0$, and there exists a constant $m \in (0, 1)$ such that the condition $\langle d_t, g_t \rangle \geq m ||d_t||^2$ holds for all $t \geq 1$ where $d_t \neq 0$, then the norm of the search directions converges to zero:
    \begin{equation*}
        ||d_t|| \leq \frac{C}{(1-m)\sqrt{t+1}} \to 0,
    \end{equation*}
    which implies $||d_t|| \to 0$ as $t \to \infty$.
\end{enumerate}
\end{theorem}

The proof can be seen in~\cite{K1998}.

\begin{theorem} \label{d_k and g_k}
     \label{d_k^*}
    Let $t$ be the smallest index for which $||d_t|| \leq \varepsilon$ and $||d_{t-1}|| > \sqrt{1 + \eta} \cdot \varepsilon $, with the additional assumption that $\langle g_t,d_{t-1} \rangle \geq 0$, then we have $||g_t|| \leq \sqrt{1+\frac{1}{\eta}} \cdot \varepsilon$.
\end{theorem}

\begin{proof}
    Suppose the claim is false, then $||g_t||^2 > (1+\frac{1}{\eta}) \cdot \varepsilon^2$. Thus,
    \begin{equation} \label{norm of d_k}
        \begin{aligned}
            ||d_t||^2 & = ||\lambda_t^* g_t + (1-\lambda_t^*)d_{t-1}||^2\\
            & = (\lambda_t^*)^2||g_t||^2+(1-\lambda_t^*)^2||d_{t-1}||^2 + 2\lambda_t^*(1-\lambda_t^*)\langle g_t,d_{t-1} \rangle \\
            & \geq (\lambda_t^*)^2||g_t||^2+(1-\lambda_t^*)^2||d_{t-1}||^2 \\
            & > [(1+\frac{1}{\eta}) \cdot (\lambda_t^*)^2 + (1+\eta) \cdot (1-\lambda_t^*)^2 ] \cdot \varepsilon^2,
        \end{aligned}
    \end{equation}

\noindent where the first inequality holds because of the additional assumption. To prove that $||d_t|| > \varepsilon$, it suffices to show that 

\begin{equation*}
    (1+\frac{1}{\eta}) \cdot (\lambda_t^*)^2 + (1+\eta) \cdot (1-\lambda_t^*)^2 \geq 1.
\end{equation*}

\noindent Note that 

\begin{equation*}
    \begin{aligned}
        p(\lambda_t^*) \overset{\Delta}{=} & (1+\frac{1}{\eta}) \cdot (\lambda_t^*)^2 + (1+\eta) \cdot (1-\lambda_t^*)^2 - 1 \\
        = &(2+\frac{1}{\eta} + \eta) \cdot (\lambda_t^*)^2 -  2(1+\eta) \lambda_t^* + \eta, 
    \end{aligned}   
\end{equation*}

\noindent is a quadratic function with respect to $\lambda_t^*$ and its discriminant is $0$. Therefore, for any $\lambda_t^* \in [0,1]$, we have $p(\lambda_t^*) \geq 0$. This implies $||d_t|| > \varepsilon $. However, this contradicts the hypothesis that $||d_t|| \leq \varepsilon $ in the lemma.  
\end{proof}

\subsection{Sample Complexity Analysis}

In this subsection, we aim to give a guidance of how to choose the sample size so that the function $L_t$ will serve as an “effective” stochastic local approximation of $L$.

\begin{theorem} \label{sample complexity}
Assume that $L$ and $L_t$ are Lipschitz functions with constant $L_f$, for any $0 < \varepsilon <1$, $\kappa > \frac{4L_f}{\delta_{min}}$. If for all $t$ and $\theta$, $|L_t(\theta)| \leq M$ and 
\begin{equation} \label{sample_size}
 |N_t| \geq -8\log(\varepsilon/2)\cdot \frac{(M+1)^2}{\kappa^2\delta_t^4},    
\end{equation}
then
\begin{equation*}
    \mathbb{P}(|L_t(\theta)-L(\theta)| \leq \kappa \delta_t^2, \ \forall \theta \in \mathcal{B}(\hat{\theta}_t,\delta_t) ) \geq 1-\varepsilon.
\end{equation*}
\end{theorem}

The proof of Theorem \ref{sample complexity} is shown in Appendix \ref{appendix}.

\section{Experiments} \label{experiments}

\subsection{Experimental Setup}

The experiments aim to compare the performance of different optimization algorithms on training LLM based on LSTM (which is widely used in many different areas (\cite{feng2025sca,li2025mitigating,litabular,zeng2025futuresightdrive})). The optimizers included in the study are Adam, AdamW and SCSAdamW. The comparison is conducted by training the language model Wikitext-2~\cite{merity2016pointer}, PennTreebank~\cite{marcus1993building}, ag-news~\cite{zhang2015character} and imbd~\cite{maas2011learning}, measuring the loss across multiple epochs, and analyzing the convergence behavior of each optimizer.

The language model is implemented as an LSTM-based neural network with an embedding layer, a four-layer LSTM with a hidden size of 256, an attention layer and a fully connected output layer mapping to a vocabulary size of 1000. Both input $x$ and target $y$ sequences have a batch size of 32 and a sequence length of 20. The loss function used is cross-entropy loss, which is standard for classification-based language modeling tasks.

The training process is conducted over 200 epochs for each optimizer. To ensure a fair comparison, the model is reinitialized with the same weights before training with each optimizer. Each optimizer is configured with a learning rate of 0.001 and a weight decay of 0.001. The hidden states of the LSTM are reset and detached at each epoch to prevent gradient explosion. Throughout training, the loss values are recorded for each epoch, and performance is evaluated by comparing the loss curves across optimizers.

\subsection{Preliminary Experimental Results}
The algorithms of this paper were implemented on a MacBook Pro 2023 with a 2.6GHz 12-core Apple M3 Pro processor and 32GB of 2400MHz DDR4 onboard memory. The code used in this research is written in Python and is available at the following GitHub repository: \href{https://github.com/yhz0/scs-experiments}{SCSAdamW} (accessed on April 24, 2025), and the computational results are shown in Figure \ref{scs pha}.

\begin{figure}[H]
     \centering
     \caption{Objective function values for different algorithms.}
     \begin{subfigure}[b]{0.45\textwidth}
         \centering
        \includegraphics[width=\textwidth]{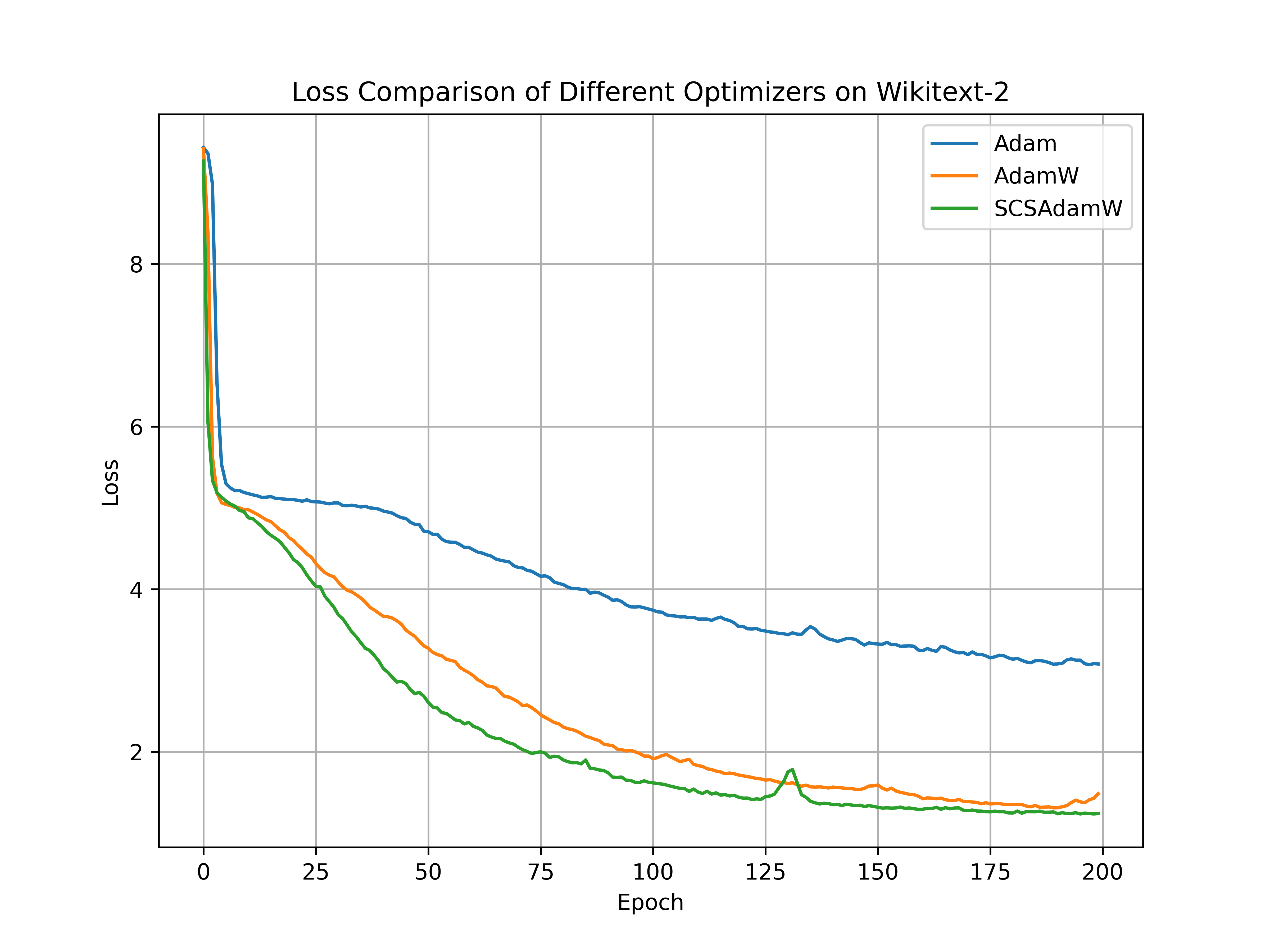}
         \caption{Wikitext-2}
     \end{subfigure}
     \hfill
     \begin{subfigure}[b]{0.45\textwidth}
         \centering
         \includegraphics[width=\textwidth]{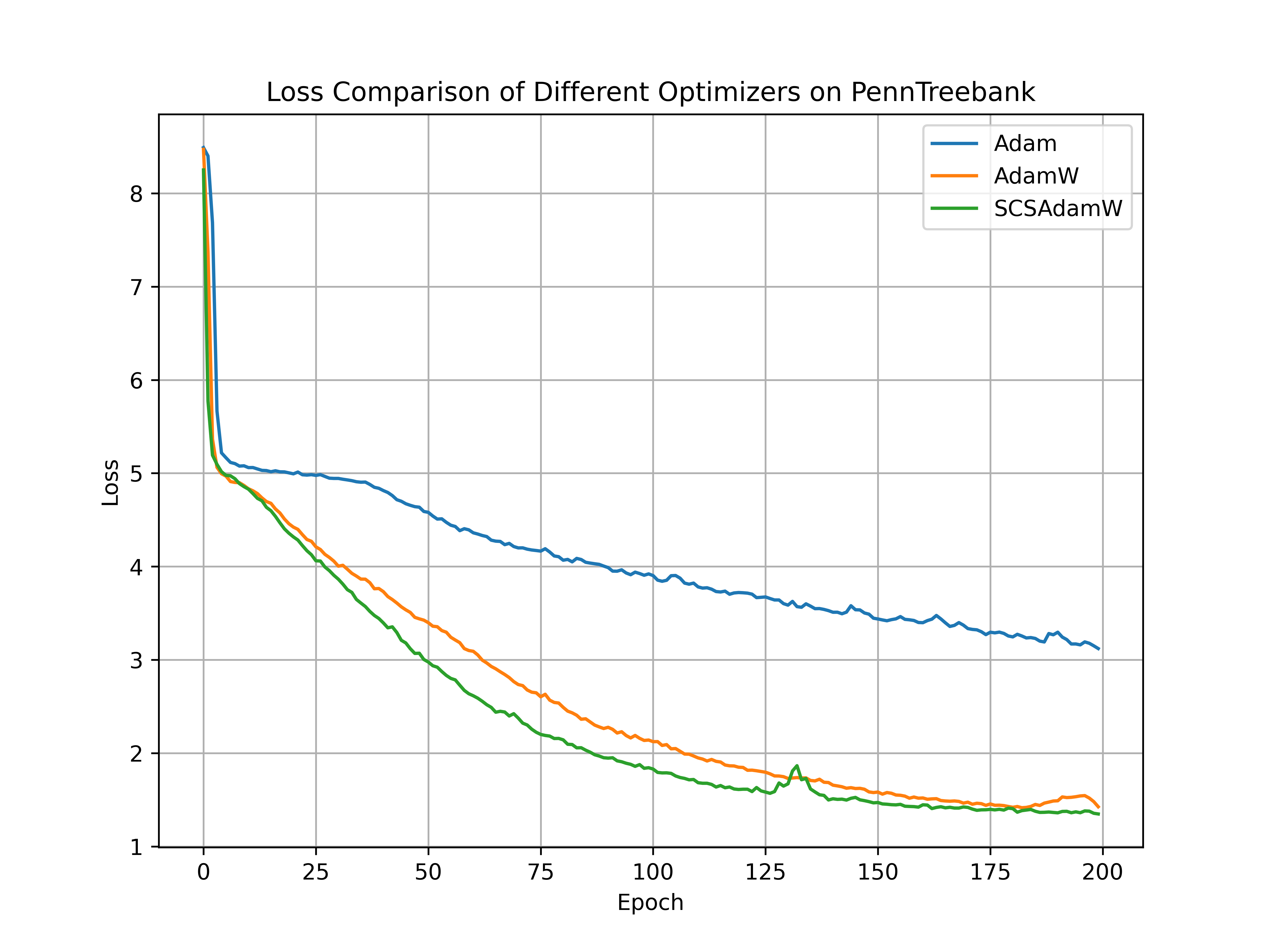}
         \caption{PennTreebank}
     \end{subfigure} \\
     \begin{subfigure}[b]{0.45\textwidth}
         \centering
         \includegraphics[width=\textwidth]{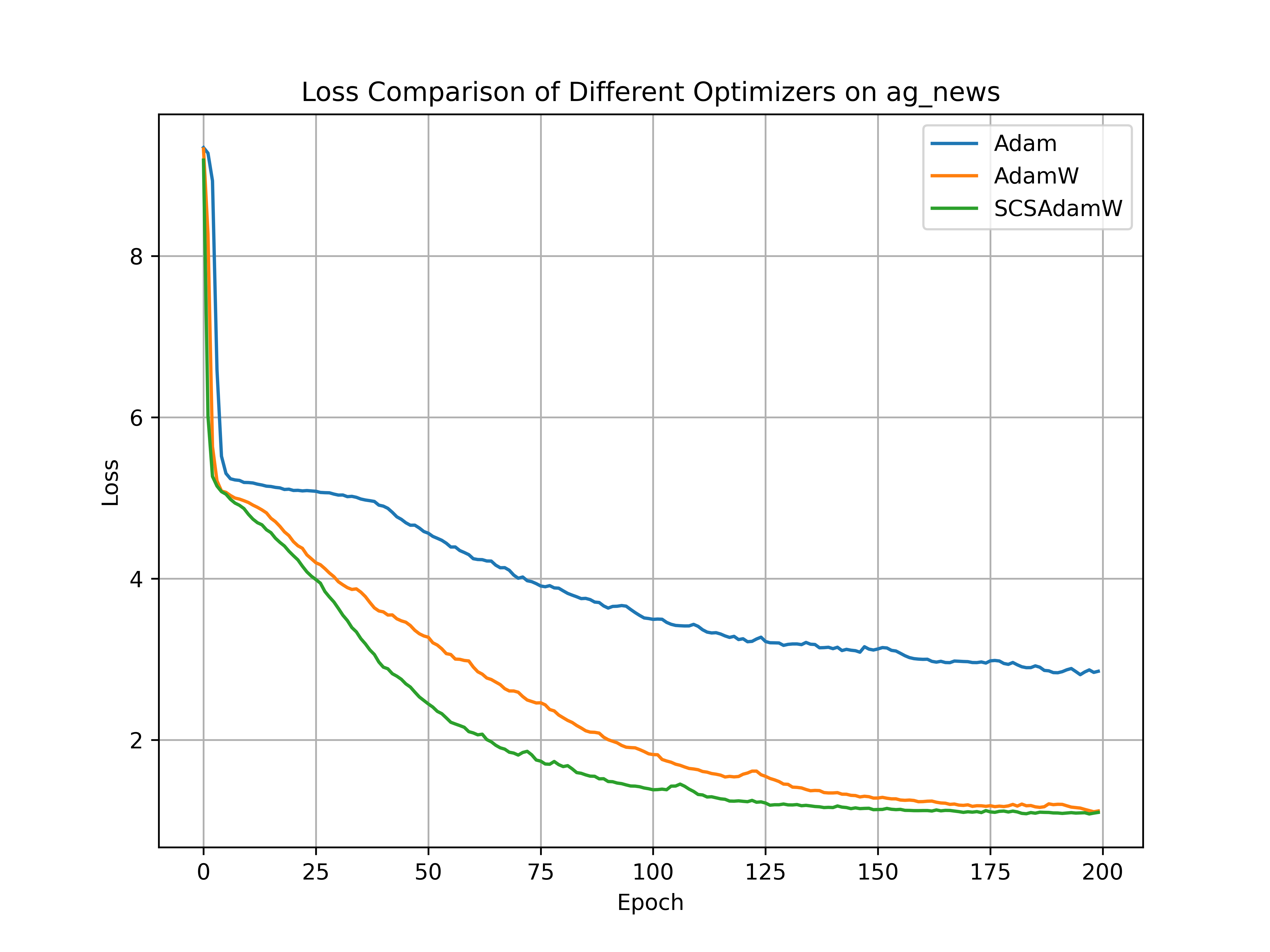}
         \caption{Ag-news}
     \end{subfigure}
     \hfill
     \begin{subfigure}[b]{0.45\textwidth}
         \centering
         \includegraphics[width=\textwidth]{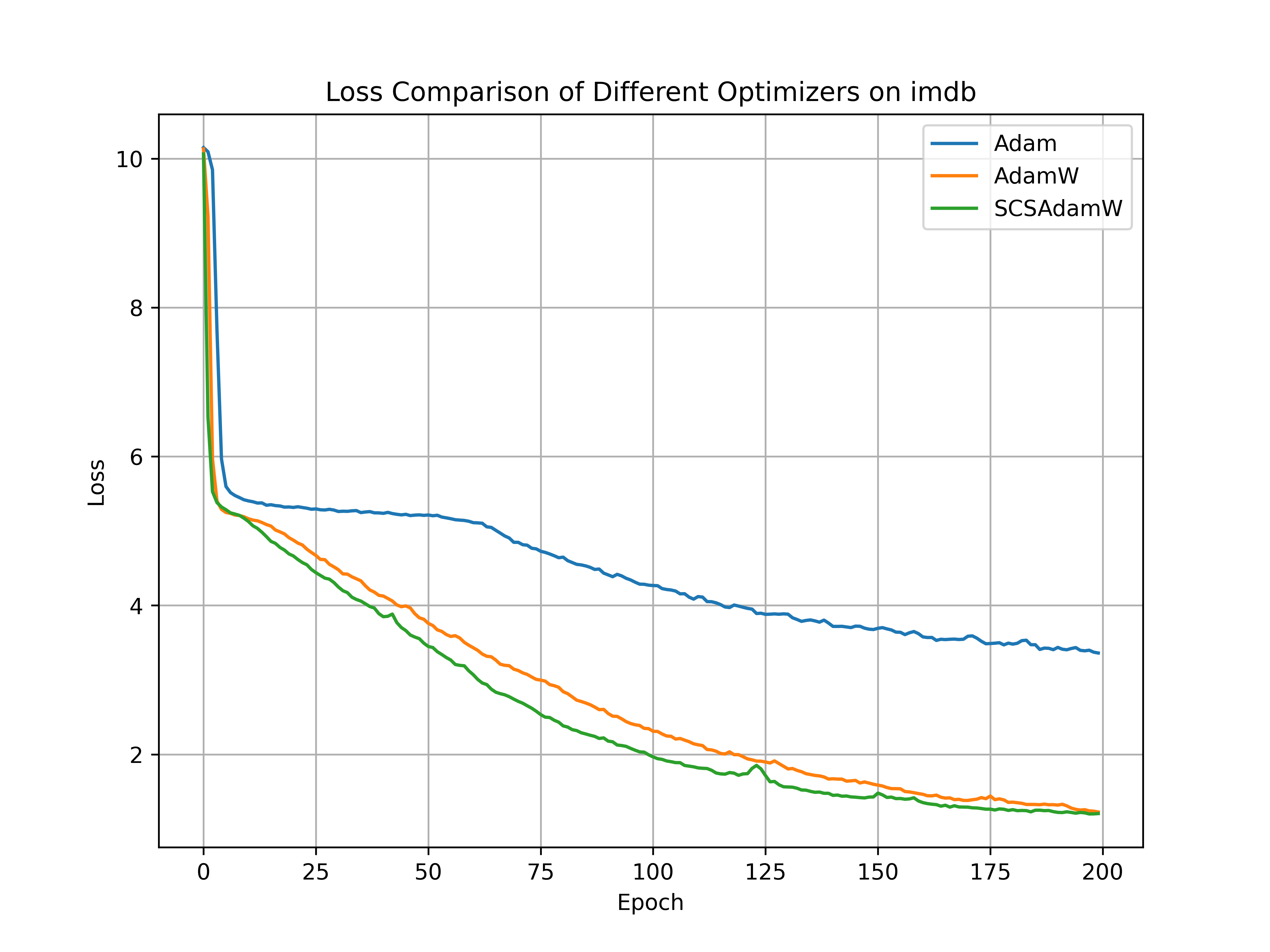}
         \caption{Imdb}
     \end{subfigure}
        \label{scs pha}
\end{figure}

\noindent \textbf{Remark}: Compared with the benchmark (Adam and AdamW) algorithms for training LLMs, the objective function values of SCSAdamW decrease the fastest, reflecting the good properties demonstrated in \textsection\ref{discussion}.    

\section{Limitations and Future Work} \label{limitations}

Although the SCSAdamW algorithm demonstrates strong performance both theoretically and empirically, several limitations remain, as outlined below:

\begin{itemize}
    \item \textbf{Smooth Direction Update.} As shown in Figure~\ref{scs pha}, while the proposed method achieves faster decreases in the objective function, it also exhibits greater fluctuations. This is primarily because, unlike AdamW, which uses a momentum term with decay rate $\beta_1$ to smooth the update direction, our algorithm relies on $\lambda_t^*$, which can become unstable—especially when approaching zero. To address this, a smoother update rule such as
    \[
    \lambda_t^* = \sigma\left(\text{clip}(\lambda_t^*, -5, 5)\right)
    \]
    may help stabilize the updates, where $\sigma(\cdot)$ denotes the sigmoid function.

    \item \textbf{Periodic Restarts.} The SCS direction accumulates information from previous subgradients, which can sometimes mislead the current search direction if outdated or irrelevant. To mitigate this, a periodic restart mechanism may be beneficial to reset the direction and avoid stagnation or divergence.

    \item \textbf{Scalability to Larger Instances and LLMs.} Due to computational resource constraints, our preliminary experiments are limited to relatively small datasets and language models. Future work will explore larger-scale benchmarks using GPU acceleration to better assess the scalability and robustness of the proposed method.
\end{itemize}

\section{Conclusion} \label{conclusion}

This paper introduced SCSAdamW, an optimization algorithm can be used for training Large Language Models (LLMs). Our approach presents several key innovations compared to conventional SGD-based benchmark algorithms:

\begin{enumerate} \item SCSAdamW employs a stochastic conjugate subgradient direction, as opposed to standard subgradient directions. The advantageous theoretical properties of this direction, discussed in \textsection\ref{discussion}, allow it to incorporate more information about the loss landscape, aiming for faster convergence than typical first-order methods.

\item While the conjugate subgradient approach seeks to capture benefits often associated with higher-order methods (such as improved curvature handling), SCSAdamW maintains a computational resource requirement comparable to that of standard SGD methods. Specifically, unlike true second-order methods that necessitate the computation and storage of a Hessian matrix, our algorithm only requires access to the search direction from the previous iteration.

\item By leveraging sample complexity analysis and concentration inequalities, SCSAdamW dynamically adjusts the sample size during training. This adaptive approach can offer improved robustness and efficiency compared to traditional training paradigms that rely on fixed sample sizes (fixed-batch SAA) throughout optimization.

\end{enumerate}

In addition to its promising theoretical underpinnings, our preliminary computational results highlight the algorithm's strong empirical performance across several test instances. On these benchmarks, SCSAdamW demonstrated faster convergence and achieved lower objective function values compared to widely used optimizers such as Adam and AdamW. Given the broad applicability of LLMs, we believe our algorithm will prove valuable in a wide range of practical scenarios~\cite{song2024research,song2025research}.

\bibliographystyle{abbrv}
\bibliography{main} 

\begin{thebibliography}{10}

\bibitem{CScheinberg2017}
F.~E. Curtis and K.~Scheinberg.
\newblock Optimization methods for supervised machine learning: From linear models to deep learning.
\newblock In {\em Leading Developments from INFORMS Communities}, pages 89--114. INFORMS, 2017.

\bibitem{DS2024}
S.~Diao and S.~Sen.
\newblock A reliability theory of compromise decisions for large-scale stochastic programs.
\newblock {\em arXiv preprint arXiv:2405.10414}, 2024.

\bibitem{feng2025sca}
C.~Feng, B.~Ba{\v{c}}i{\'c}, and W.~Li.
\newblock Sca-lstm: A deep learning approach to golf swing analysis and performance enhancement.
\newblock In {\em International Conference on Neural Information Processing}, pages 72--86. Springer, 2025.

\bibitem{HS1991}
J.~L. Higle and S.~Sen.
\newblock Stochastic decomposition: An algorithm for two-stage linear programs with recourse.
\newblock {\em Mathematics of operations research}, 16(3):650--669, 1991.

\bibitem{HS1994}
J.~L. Higle and S.~Sen.
\newblock Finite master programs in regularized stochastic decomposition.
\newblock {\em Mathematical Programming}, 67(1):143--168, 1994.

\bibitem{J2018}
X.-B. Jin, X.-Y. Zhang, K.~Huang, and G.-G. Geng.
\newblock Stochastic conjugate gradient algorithm with variance reduction.
\newblock {\em IEEE Transactions on Neural Networks and Learning Systems}, 30(5):1360--1369, 2018.

\bibitem{kingma2014adam}
D.~P. Kingma and J.~Ba.
\newblock Adam: A method for stochastic optimization.
\newblock {\em arXiv preprint arXiv:1412.6980}, 2014.

\bibitem{K1998}
I.~V. Konnov.
\newblock A combined relaxation method for variational inequalities with nonlinear constraints.
\newblock {\em Mathematical Programming}, 80(2):239--252, 1998.

\bibitem{li2025mitigating}
Z.~Li and Z.~Ke.
\newblock Mitigating demographic bias in vision transformers via attention-guided fair representation learning.
\newblock In {\em Workshop on Demographic Diversity in Computer Vision@ CVPR 2025}.

\bibitem{litabular}
Z.~Li, B.~Wang, and Z.~Ke.
\newblock From tabular to time series: Can tabpfn handle mixed data? a study on physionet.
\newblock In {\em 1st ICML Workshop on Foundation Models for Structured Data}.

\bibitem{lin2024economic}
M.~Lin, D.~Zhang, B.~Chen, and H.~Zheng.
\newblock The economic analysis of the common pool method through the hara utility functions.
\newblock {\em arXiv preprint arXiv:2408.05194}, 2024.

\bibitem{loshchilov2017decoupled}
I.~Loshchilov and F.~Hutter.
\newblock Decoupled weight decay regularization.
\newblock {\em arXiv preprint arXiv:1711.05101}, 2017.

\bibitem{maas2011learning}
A.~L. Maas, R.~E. Daly, P.~T. Pham, D.~Huang, A.~Y. Ng, and C.~Potts.
\newblock Learning word vectors for sentiment analysis.
\newblock {\em Proceedings of the 49th Annual Meeting of the Association for Computational Linguistics: Human Language Technologies}, 1:142--150, 2011.

\bibitem{M1999}
W.-K. Mak, D.~P. Morton, and R.~K. Wood.
\newblock Monte carlo bounding techniques for determining solution quality in stochastic programs.
\newblock {\em Operations research letters}, 24(1-2):47--56, 1999.

\bibitem{marcus1993building}
M.~Marcus, B.~Santorini, and M.~A. Marcinkiewicz.
\newblock Building a large annotated corpus of english: The penn treebank.
\newblock {\em Computational linguistics}, 19(2):313--330, 1993.

\bibitem{MJ2010}
J.~Martens et~al.
\newblock Deep learning via hessian-free optimization.
\newblock In {\em ICML}, volume~27, pages 735--742, 2010.

\bibitem{merity2016pointer}
S.~Merity, C.~Xiong, J.~Bradbury, and R.~Socher.
\newblock Pointer sentinel mixture models.
\newblock {\em arXiv preprint arXiv:1609.07843}, 2016.

\bibitem{N2006}
J.~Nocedal and S.~Wright.
\newblock {\em Numerical Optimization}.
\newblock Springer Science \& Business Media, 2006.

\bibitem{polyak1964some}
B.~T. Polyak.
\newblock Some methods of speeding up the convergence of iteration methods.
\newblock {\em Ussr computational mathematics and mathematical physics}, 4(5):1--17, 1964.

\bibitem{sen2022stochastic}
S.~Sen.
\newblock A stochastic conjugate subgradient algorithm for kernelized support vector machines: The evidence.
\newblock {\em NeurIPS Workshop}, 2022.

\bibitem{S2016}
S.~Sen and Y.~Liu.
\newblock Mitigating uncertainty via compromise decisions in two-stage stochastic linear programming: Variance reduction.
\newblock {\em Operations Research}, 64(6):1422--1437, 2016.

\bibitem{Shapiro2003}
A.~Shapiro.
\newblock Monte carlo sampling methods.
\newblock {\em Handbooks in operations research and management science}, 10:353--425, 2003.

\bibitem{S1998}
A.~Shapiro and T.~Homem-de Mello.
\newblock A simulation-based approach to two-stage stochastic programming with recourse.
\newblock {\em Mathematical Programming}, 81(3):301--325, 1998.

\bibitem{song2024research}
J.~Song, K.~Ding, R.~Cheng, X.~Luo, Y.~Liu, Y.~Tian, and Z.~Duan.
\newblock Research on the internationalization process of chinese hotel management groups—taking jinjiang hotel group as an example.
\newblock {\em Open Journal of Business and Management}, 13(1):41--48, 2024.

\bibitem{song2025research}
J.~Song, K.~Ding, R.~Cheng, X.~Zhao, X.~Luo, Y.~Liu, Y.~Tian, and Z.~Duan.
\newblock Research on brand strategy of hotel enterprises—taking hyatt hotel group as an example.
\newblock {\em Open Journal of Business and Management}, 13(2):861--869, 2025.

\bibitem{V1998}
V.~Vapnik.
\newblock Statistical learning theory.
\newblock {\em John Wiley \& Sons google schola}, 2:831--842, 1998.

\bibitem{W1974}
P.~Wolfe.
\newblock Note on a method of conjugate subgradients for minimizing nondifferentiable functions.
\newblock {\em Mathematical Programming}, 7:380--383, 1974.

\bibitem{W1975}
P.~Wolfe.
\newblock A method of conjugate subgradients for minimizing nondifferentiable functions.
\newblock In {\em Nondifferentiable optimization}, pages 145--173. Springer, 1975.

\bibitem{Y2022}
Z.~Yang.
\newblock Adaptive stochastic conjugate gradient for machine learning.
\newblock {\em Expert Systems with Applications}, page 117719, 2022.

\bibitem{Y2016}
F.~X.~X. Yu, A.~T. Suresh, K.~M. Choromanski, D.~N. Holtmann-Rice, and S.~Kumar.
\newblock Orthogonal random features.
\newblock {\em Advances in neural information processing systems}, 29:1975--1983, 2016.

\bibitem{zeng2025futuresightdrive}
S.~Zeng, X.~Chang, M.~Xie, X.~Liu, Y.~Bai, Z.~Pan, M.~Xu, and X.~Wei.
\newblock Futuresightdrive: Thinking visually with spatio-temporal cot for autonomous driving.
\newblock {\em arXiv preprint arXiv:2505.17685}, 2025.

\bibitem{zhang2024stochastic}
D.~Zhang.
\newblock {\em A Stochastic Conjugate Subgradient Framework for Large-Scale Stochastic Optimization Problems}.
\newblock PhD thesis, University of Southern California, 2024.

\bibitem{zhang2024sampling}
D.~Zhang and S.~Sen.
\newblock A sampling-based progressive hedging algorithm for stochastic programming.
\newblock {\em arXiv preprint arXiv:2407.20944}, 2024.

\bibitem{ZS2024}
D.~Zhang and S.~Sen.
\newblock The stochastic conjugate subgradient algorithm for kernel support vector machines.
\newblock {\em arXiv preprint arXiv:2407.21091}, 2024.

\bibitem{zhang2025stochastic}
D.~Zhang and S.~Sen.
\newblock A stochastic conjugate subgradient algorithm for two-stage stochastic programming.
\newblock {\em arXiv preprint arXiv:2503.21053}, 2025.

\bibitem{zhang2015character}
X.~Zhang, J.~Zhao, and Y.~LeCun.
\newblock Character-level convolutional networks for text classification.
\newblock In {\em Advances in Neural Information Processing Systems}, 2015.

\end{thebibliography}

\appendix

\section{Technical Appendices and Supplementary Material} \label{appendix}

Proof of Theorem \ref{sample complexity} is given bellow.
\begin{proof}
First, at the point $\theta_t$, let 
\begin{equation*}
    \begin{aligned}
        & L(\theta_t)=-E [ \sum_{t=1}^{T_i} \log P_\theta(y_t^{(i)} \mid y_{1:t-1}^{(i)})],\\
        & L_t(\theta_t)= - \frac{1}{N_t}
         \sum_{i=1}^{N_t} \sum_{t=1}^{T_i} \log P_\theta(y_t^{(i)} \mid y_{1:t-1}^{(i)}).
    \end{aligned}
\end{equation*}

\noindent Thus, 

\begin{equation*}
        L_t(\theta_t)-L(\theta_t)=E [ \sum_{t=1}^{T_i} \log P_\theta(y_t^{(i)} \mid y_{1:t-1}^{(i)})]- \frac{1}{N_t}
         \sum_{i=1}^{N_t} \sum_{t=1}^{T_i} \log P_\theta(y_t^{(i)} \mid y_{1:t-1}^{(i)}).
\end{equation*}
\noindent Using the assumption that $|Lt(\theta_t)| \leq M$ for all $t$, Hoeffding's inequality implies that
\begin{equation*}
    \mathbb{P}(|L_t(\theta_t)-L(\theta_t)| \leq \frac{1}{2}\kappa \delta_t^2) \geq 1-2\exp \left (-\frac{\kappa^2\delta_t^4 N_t^2}{8 \cdot N_t \cdot (M+1)^2} \right ) \geq 1-\varepsilon,
\end{equation*}
\noindent which indicates that when 
\begin{displaymath}
   N_t \geq -8\log(\varepsilon/2)\cdot \frac{(M+1)^2}{\kappa^2\delta_t^4}, 
\end{displaymath} 
we have 
\begin{equation*}
    \mathbb{P}(|L_t(\theta_t)-L(\theta_t)| \leq \frac{1}{2}\kappa\delta_t^2 ) \geq 1-\varepsilon.
\end{equation*}

\noindent For any other $\theta \in \mathcal{B}(\theta_t,\delta_t)$, if $|L_t(\theta_t)-L(\theta_t)| \leq \frac{1}{2}\kappa\delta_t^2$, then

\begin{equation*}
    \begin{aligned}
        |L_t(\theta)-L(\theta)| & \leq |L_t(\theta)-L_t(\theta_t)|+|L_t(\theta_t)-L(\theta_t)|+ |L(\theta_k)-L(\theta)| \\
        & \leq 2L_f \cdot \delta_t + \frac{1}{2} \kappa \delta_t^2\\
        & \leq \kappa \delta_t^2,
    \end{aligned}
\end{equation*}

\noindent where the second last inequality is due to the assumption that $L_t$ and $L$ are Lipschitz functions and the last inequality is because $\kappa>\frac{4L_f}{\delta_{min}}>\frac{4L_f}{\delta_t}$. Thus, we conclude that if $N_t$ satisfies Equation \eqref{sample_size}, then 

\begin{equation*} \label{fl1}
    \mathbb{P}(|L_t(\theta)-L(\theta)| \leq \kappa \delta_t^2, \ \forall \theta \in \mathcal{B}(\theta_t,\delta_t) ) \geq 1-\varepsilon.
\end{equation*}
\end{proof}

\end{document}